\documentclass[11pt]{article}

\usepackage{amssymb,amsmath,bm,graphicx}
\usepackage{mathrsfs}
\usepackage{constants}
\usepackage{enumitem}
\usepackage{textcomp}
\usepackage{url}
\usepackage{verbatim}
\usepackage{cite}
\usepackage{xcolor}
\usepackage{microtype}
\usepackage{longtable}
\usepackage{tabto}
\usepackage{amsmath}
\usepackage{float}
\usepackage{subcaption}
\usepackage{caption}
\usepackage{algpseudocode}
\usepackage{algorithm}
\numberwithin{equation}{section}

\newtheorem{thm}{Theorem}[section]

\usepackage{color}

\usepackage{float}

\usepackage{placeins}

\usepackage{amssymb}
\usepackage{graphicx}

\usepackage{subcaption}
\usepackage{caption}

\usepackage{csvsimple}
\usepackage[T2A]{fontenc}
\usepackage[utf8]{inputenc}
\usepackage{babel}
\usepackage{dblfloatfix}
\usepackage{amsmath}
\usepackage{lipsum}
\usepackage{amsfonts}
\usepackage{caption}

\usepackage{newtxmath,booktabs,sectsty}
\paragraphfont{\mdseries\itshape}
\usepackage{tabularx,ragged2e}
\newcolumntype{L}[1]{>{\RaggedRight\hsize=#1\hsize}X}
\newcolumntype{C}[1]{>{\Centering\hsize=#1\hsize\hspace{0pt}}X}

\usepackage{titlesec}

\titleformat{\paragraph}
{\normalfont\normalsize\bfseries}{\theparagraph}{1em}{}
\titlespacing*{\paragraph}
{0pt}{3.25ex plus 1ex minus .2ex}{1.5ex plus .2ex}

\date{}
\begin{document}
\newcommand{\bea}{\begin{eqnarray}}
\newcommand{\ena}{\end{eqnarray}}
\newcommand{\beas}{\begin{eqnarray*}}
\newcommand{\enas}{\end{eqnarray*}}
\newcommand{\beq}{\begin{equation}}
\newcommand{\enq}{\end{equation}}
\def\qed{\hfill \mbox{\rule{0.5em}{0.5em}}}
\newcommand{\bbox}{\hfill $\Box$}
\newcommand{\ignore}[1]{}
\newcommand{\ignorex}[1]{#1}
\newcommand{\wtilde}[1]{\widetilde{#1}}
\newcommand{\qmq}[1]{\quad\mbox{#1}\quad}
\newcommand{\qm}[1]{\quad\mbox{#1}}
\newcommand{\nn}{\nonumber}
\newcommand{\Bvert}{\left\vert\vphantom{\frac{1}{1}}\right.}
\newcommand{\To}{\rightarrow}
\newcommand{\E}{\mathbb{E}}
\newcommand{\Var}{\mathrm{Var}}
\newcommand{\Cov}{\mathrm{Cov}}
\newcommand{\Corr}{\mathrm{Corr}}
\newcommand{\dist}{\mathrm{dist}}
\newcommand{\diam}{\mathrm{diam}}
\makeatletter
\newsavebox\myboxA
\newsavebox\myboxB
\newlength\mylenA
\newcommand*\xoverline[2][0.70]{%
    \sbox{\myboxA}{$\m@th#2$}%
    \setbox\myboxB\null
    \ht\myboxB=\ht\myboxA%
    \dp\myboxB=\dp\myboxA%
    \wd\myboxB=#1\wd\myboxA
    \sbox\myboxB{$\m@th\overline{\copy\myboxB}$}
    \setlength\mylenA{\the\wd\myboxA}
    \addtolength\mylenA{-\the\wd\myboxB}%
    \ifdim\wd\myboxB<\wd\myboxA%
       \rlap{\hskip 0.5\mylenA\usebox\myboxB}{\usebox\myboxA}%
    \else
        \hskip -0.5\mylenA\rlap{\usebox\myboxA}{\hskip 0.5\mylenA\usebox\myboxB}%
    \fi}
\makeatother

\newtheorem{theorem}{Theorem}[section]
\newtheorem{corollary}[theorem]{Corollary}
\newtheorem{conjecture}[theorem]{Conjecture}
\newtheorem{proposition}[theorem]{Proposition}
\newtheorem{lemma}[theorem]{Lemma}
\newtheorem{definition}[theorem]{Definition}
\newtheorem{example}[theorem]{Example}
\newtheorem{remark}[theorem]{Remark}
\newtheorem{case}{Case}[section]
\newtheorem{condition}{Condition}[section]
\newcommand{\proof}{\noindent {\it Proof:} }

\title{{\bf\Large A Bayesian cluster validity index}}
\author{Nathakhun Wiroonsri\thanks{This author is financially supported by National Research Council of Thailand (NRCT), Grant number: N42A660991 (2023). Email: nathakhun.wir@kmutt.ac.th} and Onthada Preedasawakul \thanks{Email: o.preedasawakul@gmail.com }  \\ Mathematics and Statistics with Applications Research Group (MaSA) \\ Department of Mathematics, King Mongkut's University of Technology Thonburi}


\maketitle

\begin{abstract}
Selecting the appropriate number of clusters is a critical step in applying clustering algorithms. To assist in this process, various cluster validity indices (CVIs) have been developed. These indices are designed to identify the optimal number of clusters within a dataset. However, users may not always seek the absolute optimal number of clusters but rather a secondary option that better aligns with their specific applications. This realization has led us to introduce a Bayesian cluster validity index (BCVI), which builds upon existing indices. The BCVI utilizes either Dirichlet or generalized Dirichlet priors, resulting in the same posterior distribution. We evaluate our BCVI using the Wiroonsri index for hard clustering and the Wiroonsri–Preedasawakul index  for soft clustering as underlying indices. We compare the performance of our proposed BCVI with that of the original underlying indices and several other existing CVIs, including Davies–Bouldin, Starczewski, Xie–Beni, and KWON2 indices. Our BCVI offers clear advantages in situations where user expertise is valuable, allowing users to specify their desired range for the final number of clusters. To illustrate this, we conduct experiments classified into three different scenarios. Additionally, we showcase the practical applicability of our approach through real-world datasets, such as MRI  brain tumor images. These tools will be published as a new R package `BayesCVI'. 
\end{abstract}




\section{Introduction} \label{sec:introduction}

Cluster analysis is a well-known unsupervised learning tool in statistical and machine learning. It is used to split observations into groups with similar behaviors (refer to the book by \cite{statbook2023} for a review). Researchers apply cluster analysis to solve problems in various fields, ranging from social science to outer space. There are different types of clustering algorithms, including centroid-based clustering (such as K-means and fuzzy c-means (FCM)), hierarchical clustering (which includes single linkage, complete linkage, group average agglomerative, and Ward's criterion), density-based clustering (such as DBSCAN, DENCLUE, and OPTICS), probabilistic clustering (such as EM), grid-based clustering (such as CLIQUE, MAFIA, ENCLUS, and OptiGrid), and spectral clustering (see \cite{DataClusteringBook2014} for more detailed information on these techniques). Recently, there has been significant attention given to deep learning clustering \cite{ClusteringWithDeepLearning2018} and 3D point cloud clustering \cite{3Dpointclouds2020,3Dpointclouds2021,3Dpointclouds2023}. Some examples of 3D point cloud techniques include PointNet, PointNet++, DGCNN, and RandLA-Net.

Before the actual clustering process, most clustering algorithms require a necessary step known as clustering tendency assessment. This pre-clustering process aims to determine the existence of clusters in a dataset and, if present, to identify an appropriate unknown number of clusters (refer to \cite{EvaluationandValid, Clusteringtendency2020} for more details). A cluster validity index (CVI) is commonly employed to perform this task. 

In this work, we introduce a new concept of CVI that can be applied to any clustering algorithm, along with compatible CVIs. However, we specifically focus on two classic hard and soft clustering algorithms: K-means and FCM,   to illustrate our idea. While it is impractical to list all existing CVIs here, \cite{DB1979,STR2017} are notable examples for hard clustering, and \cite{XB1991,KWON2021} for soft clustering (refer to \cite{compare} and references therein for more examples). Recently, \cite{WIR2023} and \cite{WP2023} introduced correlation-based CVIs called the Wiroonsri index (WI) and Wiroonsri–Preedasawakul index (WP), which can accurately detect the optimal number of clusters and provide information about suboptimal numbers, enabling users to rank several options. These two indices are compatible with hard and soft clustering methods, respectively. The key concept in both works is allowing users to select the final number of clusters at a local peak rather than the global one, based on their experience and application needs. This concept motivated us to integrate Bayesian principles with these indices and other existing ones.
	 
\textbf{Why experience matters:} In domains like business and healthcare, researchers possess valuable insights to estimate the number of clusters they expect. For example, one may anticipate the number of customer segments or types of lung cancer to range from 3 to 6, or understand that 2 or 3 colors are sufficient for visualizingMagnetic Resonance Imaging (MRI) images. Traditional CVIs might detect an optimal number of clusters that fall outside this expected range. Bayesian analysis, a statistical approach widely used in various fields including medical diagnosis, is particularly relevant. Bayesian deep learning,  introduced by \cite{BayeMed}, has found success in healthcare applications such as disease diagnostics, medical imaging, clinical signal processing, and electronic health records. Additionally, \cite{Bayemed1} presented a novel approach for estimating kinetic parameters in DCE-MRI using adaptive Gaussian Markov random fields. Bayesian analysis also holds significance in business and marketing contexts (see \cite{BayeBussiness2,BayeBussiness3,BayeBussiness} for examples). 

\textbf{When experience matters:} We propose a new Bayesian cluster validity index (BCVI), which is based on existing CVIs. We offer the option to use either a Dirichlet prior or a generalized Dirichlet (GD) prior for determining the optimal number of cluster candidates. This allows users to set parameters according to their experience in specific contexts. The posterior distribution remains unchanged, but the parameters are adjusted based on the data, allowing the final optimal number of clusters to be determined by considering both the data and knowledge/experience. 	

The Bayesian concept has previously been used in cluster analysis, where Bayesian clustering is used to group data points based on similarities by defining a prior distribution on the unknown partition (see \cite{BayesianClus} for a comprehensive review). One of the main strengths of Bayesian clustering is its ability to handle complex data structures and dependencies among clusters by integrating prior knowledge and assumptions into the modeling process. Furthermore, Bayesian methods can automatically determine the number of clusters by estimating the posterior distribution on the partition. Probabilistic CVIs have been studied for over two decades. \cite{DefPBidx} discussed probabilistic CVIs for normal mixtures, identifying five broad types including likelihood-based criteria, information-based criteria, Bozdogan's entropic complexity criteria, minimum information ratios, and other miscellaneous indices (see \cite{DefPBidx} and references therein for further details). Additionally, \cite{BayeCVIs} introduced a probabilistic CVI compatible with iterative Bayesian fuzzy clustering. To the best of our knowledge, there is currently a lack of literature directly discussing Bayesian cluster validity indices, which further motivates the development of our BCVI.

Our BCVI is defined based on the assumption that a ratio, which is derived from an underlying index as defined in Section \ref{sec:main}, follows a multinomial distribution given ${\bf p}$. Here, ${\bf p}$ follows a prior Dirichlet or GD distribution (${\bf p} = (p_1,p_2,\ldots,p_K)$, representing the probability that the actual number of groups is $k$ for $k=1,2,\ldots,K$). This allows users to set the parameters of the prior distribution based on their knowledge and intentions. As a result, the final number of clusters typically falls within the expected range, which is more relevant for users. We assess the performance of the BCVI with the underlying WI and WP through K-means and FCM on simulated datasets categorized into three cases, as well as real-world datasets including MRI images. Our evaluation presents the performance of BCVI alongside the original WI and WP, as well as four additional existing CVIs: Davies–Bouldin (DB), Starczewski (STR), Xie–Beni (XB), and KWON2 indices outlined in the subsequent section. Although this is not a direct comparison test due to the novelty of the concept, it serves to confirm the claimed benefits of BCVI.

The remaining sections of this work are organized as follows. Section \ref{sec:background} offers necessary background information about cluster algorithms, CVIs, and probability distributions used in this study. Our proposed index, along with its mathematical properties, is presented in Section \ref{sec:main}. Section \ref{sec:exp} delves into experimental results and potential applications to MRI images. Finally, Section \ref{sec:conclusion} provides concluding remarks and discusses potential future directions.

\section{Background} \label{sec:background}

In this section, we provide the necessary definitions and background information, including clustering algorithms, cluster validity indices, and related probability distributions.

\subsection{Cluster analysis and cluster validity index} 

In this subsection, we offer brief definitions of K-means and FCM, as well as three hard and three soft CVIs, including the two recent indices introduced in \cite{WIR2023} and \cite{WP2023}, which will serve as underlying indices for our proposed BCVI in the subsequent section. 

Let $n,k,p \in \mathbb{N}$ and denote $[n] = \{1,2,\ldots,n\}$. We establish the following notations used in this work. For $i\in[n]$ and $j \in [k]$, denote 

\begin{enumerate} \label{notations}
    \item $x_i = (x_{i1},x_{i2},\ldots,x_{ip})$: Data points.
    \item $K$: Actual number of clusters.
    \item $C_j$: Set of data points in the $j^{th}$ cluster.
    \item $v_j$: $j^{th}$ cluster centroid.
    \item $v_0$: Centroid of the entire dataset.
    \item $\bar{v}$: Centroid of all $v_j$
    \item $\mu = \left(\mu_{ij}\right)$: Membership degree matrix where $\mu_{ij}$ denotes the degree to which a sample point $x_i$ belongs to $C_j$.
    \item $\| x-y \|$: Euclidean distance between $x$ and $y$.
    \item $\Corr(\cdot,\cdot)$ Correlation coefficient. In this work, we consider only the Pearson correlation.
\end{enumerate}

\subsubsection{K-means}

K-means \cite{Kmeans} is a simple yet efficient clustering algorithm. It operates by partitioning a dataset into $k$ clusters, where $k$ is a user-defined parameter. The algorithm commences by randomly initializing cluster centroids. Subsequently, it assigns each data point to the nearest centroid and updates the centroids based on the newly assigned points. This iterative process continues until the cluster centroids converge.
The objective of K-means is to minimize the within-cluster sum of squares, expressed as:
\begin{equation*}
    \sum_{j=1}^k\sum_{x \in C_j} \|x-v_j\|^2 .
\end{equation*}

\subsubsection{Fuzzy C-means}

FCM, introduced by \cite{DUNN1973} and later refined by \cite{FCM1984}, is a clustering technique utilized to group similar data points into user-specified $c$ clusters. Each data point is assigned a membership degree, indicating the degree of belongingness to each cluster. The objective of FCM is to minimize the target function:

\begin{equation*}\label{objfunc}
    \sum_{i=1}^n\sum_{j=1}^{c}\mu_{ij}^m\| {x}_i-{v}_j\|^2,   
\end{equation*}
where $m > 1$ denotes the fuzziness parameter. 

The optimization of \eqref{objfunc} begins with random initialization of centroids $v_j$. Iteratively, the membership degrees are updated according to:

\begin{equation*}\label{membership}
    \mu_{ij} = \frac{1}{\sum_{k=1}^c\left(\frac{\|{x}_i-{v}_j\|}{\|{x}_i-{v}_k\|}\right)^\frac{2}{m-1}},
\end{equation*}
and centroids are updated as follows:
\begin{equation*}\label{centroid}
    v_{j} = \frac{\sum_{i=1}^n\mu_{ij}^m x_i}{\sum_{i=1}^n\mu_{ij}^m},
\end{equation*}
for $i \in [n]$ and $j \in [c]$.
This iterative process continues until convergence is achieved.

\subsubsection{Hard cluster validity indices}
\paragraph{Davies-Bouldin index\cite{DB1979}}

The DB's measure is defined as:
\beas
\texttt{DB}(k) = \frac{1}{k}\sum_{i=1}^kR_{i,qt},
\enas
where 
\beas
R_{i,qt} = \max_{j \in [k]\textbackslash \left\{i\right\}}\left\{\frac{S_{i,q}+S_{j,q}}{M_{ij,t}}\right\},
\enas
\beas
S_{i,q} = \left(\frac{1}{|C_i|}\sum_{x \in C_i}\left\|x-v_i\right\|^q\right)^{1/q}
\enas
for $q,t \geq 1$ and $i\neq j \in [k]$, and
\beas
M_{ij,t} = \left(\sum_{s=1}^p|v_{is}-v_{js}|^t\right)^{1/t}.
\enas
The smallest value of $\texttt{DB}(k)$ indicates a valid optimal partition. 

\paragraph{Starczewski index \cite{STR2017}} 
The STR index is defined as 
\beas
\texttt{STR}(k) = [E(k) - E(k-1)][D(k+1) - D(k)],
\enas 
where
$D(k) = \dfrac{\max_{i,j = 1}^k \|v_i-v_j\|}{\min_{i,j =1}^k \|v_i-v_j\|}$, and $E(k) = \dfrac{\sum_{i = 1}^n \|x_i - v_0\|}{\sum_{j=1}^k \sum_{x \in C_j} \|x - v_j\|}$.

The largest value of $\texttt{STR}(k)$ indicates a valid optimal partition.

\paragraph{Wiroonsri index \cite{WIR2023}}
The WI index is defined as follows:
For $m \in \{2,3,\ldots,n-1\}$ and $k \in \{2,3,\ldots,m\}$,
 
\begin{flushleft} \label{nci}
\textbf{Case 1}: $\max_{2\le l \le m} \texttt{NCI1}(k) < +\infty$
\end{flushleft}
\beas
\texttt{NCI}_m(k) =  
                 \begin{cases}
                     \min_{2\le l \le m} \left\{\texttt{NCI1}(l)|\texttt{NCI1}(l)> -\infty\right\}  \text{  \ \ if \ }  \texttt{NCI1}(k) = -\infty  \\
                     \texttt{NCI1}(k)   \text{ \ otherwise},
                 \end{cases} 
\enas

\begin{flushleft}
\textbf{Case 2}: $\max_{2\le l \le m} \texttt{NCI1}(k) = +\infty$
\end{flushleft}
\beas
\texttt{NCI}_m(k) =  
                 \begin{cases}
                     
                     \min_{2\le l \le m} \left\{\texttt{NCI1}(l)|\texttt{NCI1}(l)> -\infty\right\} + \texttt{NCI2}(k)  \text{  \ \ if \ }  \texttt{NCI1}(k) = -\infty \\
										 \max_{2\le l \le m} \left\{\texttt{NCI1}(l)|\texttt{NCI1}(l)< +\infty\right\} + \texttt{NCI2}(k)  \text{  \ \ if \ }  \texttt{NCI1}(k) = +\infty \\
										 \texttt{NCI1}(k)+ \texttt{NCI2}(k)   \text{ \ otherwise},
                 \end{cases} 
\enas
where
\bea \label{nci1}
\texttt{NCI1}(k)
           = \frac{\left(\texttt{NC}(k)-\texttt{NC}(k-1)\right)\left(1-\texttt{NC}(k)\right)}{\max\{0,\texttt{NC}(k+1)-\texttt{NC}(k)\}\left(1-\texttt{NC}(k-1)\right)} 
\ena and 
\bea \label{nci2}
\texttt{NCI2}(k) = \frac{\texttt{NC}(k)-\texttt{NC}(k-1)}{1-\texttt{NC}(k-1)} - \frac{\texttt{NC}(k+1)-\texttt{NC}(k)}{1-\texttt{NC}(k)}, 
\ena
with $\texttt{NC} = \Corr(\Vec{d},\Vec{c}(k))$, $\texttt{NC}(1) = \frac{\texttt{SD}(\vec{d}_v)}{\max \vec{d}_v - \min \vec{d}_v}$. Note that we let
\bea \label{dmdef}
\vec{d}_v = (\|x_i-v_0\|)_{i \in [n]},
\ena
\bea \label{ddef}
\vec{d} = (\|x_i-x_j\|)_{i,j \in [n]}
\ena
be a vector of length ${n \choose 2}$ containing distances of all pairs of data points, and 
\beas \label{cdef}
\vec{c}(k) = (\|v_i(k)-v_j(k)\|)_{i,j \in [n]}
\enas
be a vector of the same length containing the distances of all pairs of corresponding centroids of clusters in which two points are located. The largest value of $\texttt{NCI}(k)$ indicates a valid optimal partition.

\subsubsection{Soft cluster validity indices}
\paragraph{Xie and Beni index \cite{XB1991}}
The XB index is defined as follows:
\beas
  \texttt{XB}(k) = \dfrac{\sum_{j=1}^k\sum_{i=1}^n\mu_{ij}^2\| {x}_i-{v}_j\|^2}
             {n \cdot \min_{j\neq l} \{ \| {v}_j-{v}_l\|^2 \}} .
\enas
The smallest value of $\texttt{XB}(k)$ indicates a valid optimal partition.

\paragraph{KWON2 index \cite{KWON2021}}
The KWON2 index is defined as follows:
\beas
   \texttt{KWON2}(k) = \dfrac{w_1\left[w_2\sum_{j=1}^k\sum_{i=1}^n \mu_{ij}^{2^{\sqrt{\frac{m}{2}}}}  \|{x}_i-{v}_j\|^2  + \dfrac{\sum_{j=1}^k\| {v}_j-{v}_0\|^2}{\max_j \|{v}_j-{v}_0\|^2 } + w_3 \right]}{\min_{i \neq j} \| {v}_i-{v}_j\|^2 + \frac{1}{k}+\frac{1}{k^{m-1}}}
\enas
where $w_1 = \dfrac{n-k+1}{n}$, $w_2 = \left(\dfrac{k}{k-1}\right)^{\sqrt{2}}$ and $w_3=\dfrac{nk}{(n-k+1)^2}$.\\ The smallest value of $\texttt{KWON2}(k)$ indicates a valid optimal partition.

\paragraph{Wiroonsri and Preedasawakul index\cite{WP2023}}

The WP index is defined in three cases. Let $m \in \{2,3,\ldots,n-1\}$ and $k \in \{2,3,\ldots,m\}$,

\begin{flushleft}
\textbf{Case 1}: $\max_{2\le l \le p} \texttt{WPCI1}(k) < +\infty$ and $\exists l \in [p]\backslash \{1\}$ such that $|\texttt{WPCI1}(l)|<\infty$.
\end{flushleft}
\beas
\texttt{WP}_p(k) =  
                 \begin{cases}
                 \scriptstyle
                     \min_{2\le l \le p} \left\{\texttt{WPCI1}(l)|\texttt{WPCI1}(l)> -\infty\right\}  \text{  \ \ if \ }  \texttt{WPCI1}(k) = -\infty  \\
                     \texttt{WPCI1}(k)   \text{ \ otherwise}.
                 \end{cases} 
\enas 

\begin{flushleft}

\textbf{Case 2}: $\max_{2\le l \le p} \texttt{WPCI1}(k) = +\infty$ and $\exists l \in \{2,3,\ldots,p\}$ such that $|\texttt{WPCI1}(l)|<\infty$.
\end{flushleft}
\beas
\texttt{WP}_p(k) =  
                 \begin{cases}
                     \scriptstyle
                     \min_{2\le l \le p} \left\{\texttt{WPCI1}(l)|\texttt{WPCI1}(l)> -\infty\right\} + \texttt{WPCI2}(k)  \text{  \ \ if \ }  \texttt{WPCI1}(k) = -\infty \\
                     \scriptstyle
										 \max_{2\le l \le p} \left\{\texttt{WPCI1}(l)|\texttt{WPCI1}(l)< +\infty\right\} + \texttt{WPCI2}(k)  \text{  \ \ if \ }  \texttt{WPCI1}(k) = +\infty \\
										 \texttt{WPCI1}(k)+ \texttt{WPCI2}(k)   \text{ \ otherwise}.
                 \end{cases} 
\enas
\begin{flushleft}
\textbf{Case 3}:  $\forall l \in \{2,3,\ldots,p\}$, $|\texttt{WPCI1}(l)|=+\infty$.
\beas
\texttt{WP}_p(k) = \texttt{WPCI2}(k),
\enas
\end{flushleft}
where $\texttt{WPCI1}(k)$ and $\texttt{WPCI2}(k)$ are defined similarly to \eqref{nci1} and \eqref{nci2}, respectively, with $\texttt{NC}$ replaced by $\texttt{WPC}$ with
$\texttt{WPC}(k) = \Corr(\vec{d},\vec{\nu}(k))$, $\texttt{WPC}(1) = \frac{\texttt{SD}(\vec{d}_v)}{\max \vec{d}_v - \min \vec{d}_v}$, $\vec{d}_v$ and $\vec{d}$ are as in \eqref{ddef} and \eqref{dmdef}, respectively,
\beas 
o_i(k,\gamma) = \frac{\sum_{j=1}^k \mu_{ij}^{\gamma} v_j}{\sum_{j=1}^k \mu_{ij}^{\gamma}} \text{ \ \ and \ \ } \vec{\nu}(k) = (\|o_i(k,\gamma)-o_j(k,\gamma)\|)_{i,j \in [n]}.
\enas
The largest value of \texttt{WP(k)} indicates a valid optimal partition.

\subsection{Dirichlet  and Generalized Dirichlet distributions}

In this subsection, we state the definitions of the Dirichlet and GD distributions and their necessary properties.

\subsubsection{Dirichlet distribution}

The Dirichlet distribution \cite{Dirichlet} with parameters ${\bm \alpha}=(\alpha_1,\alpha_2,\ldots,\alpha_K)$ where $\alpha_k > 0$ has the probability density function given by
\beas
f(x_1,\ldots,x_K|{\bm \alpha}) = \frac{1}{B({\bm \alpha})} \prod_{k=1}^K x_k^{\alpha_k-1},
\enas
for $0 \le x_k \le 1$ for all $k$ and $\sum_{k=1}^K x_k = 1$ where the multivariate beta function is defined as
\beas
B({\bm \alpha}) = \frac{\prod_{k=1}^K\Gamma(\alpha_k)}{\Gamma\left(\sum_{k=1}^K\alpha_k\right)}.
\enas

The next lemma provides the mean and variance of the Dirichlet distribution.

\begin{lemma} \label{Dmoment}
    Let ${\bf X} = (X_1,\ldots,X_{K})$ has a Dirichlet distribution with parameters $\alpha_1,\ldots,\alpha_{K}$. Then
\end{lemma}
\beas \label{Dmean}
\E[X_k] = \frac{\alpha_k}{\alpha_0}, \text{ \ \ and \ \ } \label{Dvar}
Var(X_k) = \frac{\alpha_k(\alpha_0 - \alpha_k)}{\alpha_0^2(\alpha_0 +1 )},
\enas 
where 
\bea \label{alpha0}
\alpha_0 = \sum_{k=1}^K \alpha_k.
\ena

\subsubsection{Generalized Dirichlet 
 distribution}

The GD distribution \cite{WONG1998} with parameters ${\bm \alpha} = (\alpha_1,\alpha_2,\ldots,\alpha_{K-1})$, ${\bm \beta} = (\beta_1,\beta_2,\ldots,\beta_{K-1})$ where $\alpha_k,\beta_k > 0$ has the probability density function given by

\beas
f(x_1,\ldots,x_{K-1}|{\bm \alpha,\bm \beta}) =  \prod_{k=1}^{K-1} \frac{x_k^{\alpha_k-1}(1-x_1-\cdots-x_k)^{\gamma_k}}{B(\alpha_k,\beta_k)}
\enas
for $0 \le x_k\le 1$, and $x_1+x_2+\cdots+x_{K-1} \le 1$ where  $\gamma_k = \beta_k-\alpha_{k+1}-\beta_{k+1}$ for $k \in [K-2]$ and $\gamma_{K-1} = \beta_{K-1}-1$, and $B(\cdot,\cdot)$ is the beta function.

The next lemma is proved in Properties 1 and 2 in \cite{WONG1998}.

\begin{lemma} \label{moment}
    Let $s_1,\ldots,s_{K-1} \in \mathbb{N}_0$, ${\bf X} = (X_1,\ldots,X_{K-1})$ has a GD distribution with parameters $\alpha_1,\ldots,\alpha_{K-1},$ $\beta_1,\ldots,\beta_{K-1}$ and $\delta_k = \sum_{i = k+1}^{K-1} s_i$ for $k \in [K-1]$, the general moment function of ${\bf X}$ is given by
\bea \label{Gmoment}
\E[X_1^{s_1} X_2^{s_2} \ldots X_{K-1}^{s_{K-1}}] = \prod_{k=1}^{K-1} \frac{\Gamma(\alpha_k + \beta_k)\Gamma(\alpha_k + s_k)\Gamma(\beta_k+\delta_k)}{\Gamma(\alpha_k)\Gamma(\beta_k)\Gamma(\alpha_k+\beta_k+s_k+\delta_k)}.
\ena Additionally, the $sth$ moment of $X_K := 1-X_1-\cdots-X_{K-1}$ is given by
\bea \label{GKmoment}
\E[X_K^s] = \prod_{k=1}^{K-1} \frac{\Gamma(\alpha_k+\beta_k)\Gamma(\beta_k+s)}{\Gamma(\beta_k)\Gamma(\alpha_k+\beta_k+s)}.
\ena  

\end{lemma}

\section{Our proposed index} \label{sec:main}

Let ${\bf x} = (x_1,x_2,\ldots,x_n)$ denote a dataset of size $n \in \mathbb{N}$. Let $K \in \mathbb{N}$ be the maximum number of clusters to be considered, and let ${\bf p} = (p_2,p_3,\ldots,p_K)$, where $p_k$, $k=2,3,\ldots,K$ represents the probability that the dataset consists of $k$ groups. Let   
\bea \label{idxratio}
r_k(\bf x) = \begin{cases}
            \dfrac{GI(k)-\min_j GI(j)}{\sum_{i=2}^K (GI(i)-\min_j GI(j))} \text{ \ for Condition A, } \\
            \dfrac{\max_j GI(j)- GI(k)}{\sum_{i=2}^K (\max_j GI(j) - GI(i))}  \text{  \ for Condition B, } \\
        \end{cases} 
\ena
where GI represents an arbitrary CVI. 

Condition A: The largest value of the GI indicates the optimal number of clusters.

Condition B: the smallest value of the GI indicates the optimal number of clusters.

Assume that 
\bea \label{datadist}
f({\bf x}|{\bf p}) = C({\bf p}) \prod_{k=2}^Kp_k^{nr_k(x)}
\ena
represents the conditional probability density function of the dataset given ${\bf p}$, where $C({\bf p})$ is the normalizing constant.

\subsection{Prior and posterior of ${\bf p}$}

In this section, we explore two different options for the prior distribution of ${\bf p}$, namely Dirichlet and GD priors.

\subsubsection{Dirichlet prior and posterior}

Here, we assume that  ${\bf p}$ follows a Dirichlet prior distribution with parameters ${\bm \alpha} = (\alpha_2,\ldots,\alpha_K)$ with the probability density function
\bea \label{Dprior}
\pi({\bf p}) = \frac{1}{B({\bm \alpha})} \prod_{k=2}^K p_k^{\alpha_k-1}.
\ena

\begin{thm} \label{Dmain}
Let $K \in \mathbb{N}$ and ${\bf r(x)} = (r_2({\bf x}),\ldots,r_K({\bf x}))$, where $r_k({\bf x})$ is defined as in \eqref{idxratio}. Assuming that ${\bf x}$ follows the distribution described in \eqref{datadist}, the posterior distribution of ${\bf p}$ has the probability density function:
\beas
\pi({\bf p}|{\bf x}) = \frac{1}{B({\bm \alpha} + n{\bf r(x)})} \prod_{k=2}^K p_k^{\alpha_k + nr_k({\bf x})-1}.
\enas
In particular, it follows a Dirichlet distribution with parameters ${\bm \alpha}+ n{\bf r(x)}$.

\end{thm}
\begin{proof}
Starting from \eqref{datadist} and \eqref{Dprior}, the joint distribution of $({\bf x , p})$ is given by
\beas
f({\bf x , p}) = \frac{C({\bf p})}{B({\bm \alpha})} \prod_{k=2}^K p_k^{\alpha_k+nr_k({\bf x})-1}.
\enas
Integrating over ${\bf p}$ to obtain the joint probability density function, the marginal of ${\bf x}$ is 
\beas
m({\bf x}) = \frac{B({\bm \alpha}+ n{\bf r(x)} )}{B({\bm \alpha})} C({\bf p}).
\enas
Therefore, the posterior probability density function is 
\beas
\pi({\bf p}|{\bf x}) = \frac{f({\bf x , p})}{m({\bf x})} = \frac{1}{B({\bm \alpha} + n{\bf r(x)})} \prod_{k=2}^K p_k^{\alpha_k+nr_k({\bf x})-1}.
\enas
\end{proof}

The next corollary follows directly from the above theorem and Lemma \ref{Dmoment}.

\begin{corollary} \label{DirCor}
For $k = 2,3,\ldots,K$, the posterior means and variances of $p_k$ are given, respectively, by
\beas
\E[p_k|{\bf x}] = \frac{\alpha_k + nr_k({\bf x})}{\alpha_0+n},
\enas
and
\beas
\Var(p_k|{\bf x}) = \dfrac{(\alpha_k + nr_k(x))(\alpha_0 + n -\alpha_k-nr_k(x))}{(\alpha_0 + n)^2(\alpha_0 + n +1 )}.
\enas
where $\alpha_0$ is as in \eqref{alpha0} with $k$ from $2$ to $K$.

\end{corollary}

\subsubsection{Generalized Dirichlet prior and posterior}

In this subsection, we consider a GD prior distribution for ${\bf p}$, characterized by parameters ${\bm \alpha} = (\alpha_2,\ldots,\alpha_{K-1})$ and ${\bm \beta} = (\beta_2,\ldots,\beta_{K-1})$, denoted as GD$({\bm \alpha, \bm \beta})$. The probability density function of this prior distribution is given by
\bea \label{GDprior}
\pi({\bf p}) =  \prod_{k=2}^{K-1} \frac{p_k^{\alpha_k-1}(1-p_2-\cdots-p_k)^{\gamma_k}}{B(\alpha_k,\beta_k)},
\ena
where $\gamma_k = \beta_k-\alpha_{k+1}-\beta_{k+1}$ for $k = 2,3\ldots,K-2$ and $\gamma_{K-1} = \beta_{K-1}-1$.

\begin{thm} \label{GDmain}
Let $K \in \mathbb{N}$ and ${\bf r(x)} = (r_2({\bf x}),\ldots,r_K({\bf x}))$ where $r_k({\bf x})$ is as in \eqref{idxratio}. Assume that ${\bf x}$ has a distribution following \eqref{datadist} and ${\bf p}$ has a prior distribution following \eqref{GDprior}. Then the posterior distribution of ${\bf p}$ has the probability density function given by
\beas
\pi({\bf p}|{\bf x}) = \prod_{k=2}^{K-1} \frac{p_k^{\alpha'_k-1}(1-p_2-\cdots-p_k)^{\gamma'_k}}{B(\alpha'_k,\beta'_k)},
\enas
where 
\bea \label{alpha_i}
\alpha_k' = \alpha_k + nr_k(x)
\ena
and 
\bea \label{beta_i}
\beta_k' = \beta_k + \sum_{i=k+1}^K nr_i(x) 
\ena
for $k=2,\ldots,K-1$, $\gamma'_k = \beta'_k-\alpha'_{k+1}-\beta'_{k+1}$ for $k = 2,3\ldots,K-2$ and $\gamma'_{K-1} = \beta'_{K-1}-1$.

In particular, it has a GD distribution with parameters ${\bm \alpha'} = (\alpha_2',\ldots,\alpha_{K-1}')$ and ${\bm \beta'} = (\beta_2',\ldots,\beta_{K-1}')$.

\end{thm}

\begin{proof}
    From \eqref{datadist} and \eqref{GDprior}, and utilizing the fact that $p_K = 1 - p_2 - \cdots - p_{K-1}$, the joint distribution of $({\bf x , p})$ is given by

\begin{equation*}
\begin{aligned}
\begin{split}
f({\bf x , p}) &=  \frac{C({\bf p})}{\prod_{k=2}^{K-1}B(\alpha_k,\beta_k)}\prod_{k=2}^{K-2}p_k^{\alpha_k+nr_k({\bf x})-1}(1-p_2-\cdots-p_k)^{\gamma_k} \\
& \quad \times \left[p_{K-1}^{\alpha_{K-1}+nr_{K-1}(x)-1}\left(1-p_2-\cdots-p_{K-1}\right)^{\beta_{K-1}+nr_K(x)-1}\right].
\end{split}
\end{aligned}       
\end{equation*}

Since 
\beas
\gamma_k &=& \beta_k-\alpha_{k+1}-\beta_{k+1} \\ 
&=& \left(\beta_k+\sum_{i=k+1}^K nr_i(x)\right) -(\alpha_{k+1}+nr_{k+1}(x)) - \left(\beta_{k+1}+ \sum_{i=k+2}^K nr_i(x)\right), 
\enas 
for $k=2,\ldots,K-2$ and
\beas
\beta_{K-1}+nr_K(x) -1 = \left(\beta_{K-1} + \sum_{i=K}^K nr_i(x)\right) -1,
\enas
taking an integral over ${\bf p}$ to the joint probability density function, the marginal of ${\bf x}$ is 
\beas 
m({\bf x}) = C({\bf p})\prod_{k=2}^{K-1}\frac{B\left(\alpha_k+nr_k(x) ,\beta_k  +\sum_{i=k+1}^{K} nr_i(x)\right)}{{B\left(\alpha_k,\beta_k\right)}}.
\enas
Therefore, the posterior probability density function can be written as 
\beas
\pi({\bf p}|{\bf x}) = \frac{f({\bf x , p})}{m({\bf x})} = \prod_{k=2}^{K-1} \frac{p_k^{\alpha_k'}(1-p_2-\cdots-p_k)^{\gamma_k'}}{B(\alpha_k',\beta_k')},
\enas
where $\gamma_k' = \gamma_k$ for $k=2,\ldots,K-2$ , $\gamma_{K-1}' = \beta_{K-1}'-1$ and $\alpha_k'$ and $\beta_k'$, are given in \eqref{alpha_i} and  \eqref{beta_i}, respectively.
\end{proof}

Next, we compute the posterior means and variances of ${\bf p}$.

\begin{corollary} \label{GDirCor}
For $k = 2,3,\ldots,K-1$, the posterior means of $p_i$ are given by

\beas
\E[p_k|{\bf x}] = \frac{\alpha_k'}{\alpha_k'+\beta_k'} \prod_{i:i<k} \frac{\beta_i'}{\alpha_i'+\beta_i'}
\text{ \ \  and \ \  }
\E[p_K|{\bf x}] = \prod_{k=2}^{K-1} \frac{\beta_k'}{\alpha_k'+\beta_k'},
\enas
and the variances are given by
\beas
\Var(p_k|{\bf x}) = \frac{(\alpha_k'+1)\alpha_k'}{(\alpha_k'+\beta_k'+1)(\alpha_k'+\beta_k')} \prod_{i:i<k} \frac{(\beta_i'+1)\beta_i'}{(\alpha_i'+\beta_i'+1)(\alpha_i'+\beta_i')} \\
- \frac{\alpha_k'^2}{(\alpha_k'+\beta_k')^2} \prod_{i:i<k} \frac{\beta_i'^2}{(\alpha_i'+\beta_i')^2},
\enas
and
\beas
\Var(p_K|{\bf x}) = \prod_{k=2}^{K-1} \frac{(\beta_k'+1)\beta_k'}{(\alpha_k'+\beta_k'+1)(\alpha_k'+\beta_k')} - \prod_{k=2}^{K-1} \frac{\beta_k'^2}{(\alpha_k'+\beta_k')^2},
\enas
where $\alpha_k'$ and $\beta_k'$, are given in \eqref{alpha_i} and  \eqref{beta_i}, respectively.
\end{corollary}

\begin{proof}
For $k = 2,\ldots,K-1$, we apply \eqref{Gmoment} in Lemma \ref{moment} with $s_k = 1$ and $s_l=0$ for $l \ne k$, yielding

{\small
\begin{equation*}
\begin{aligned}
     \E[p_k|{\bf x}] & = \frac{\Gamma(\alpha_k' + \beta_k')\Gamma(\alpha_k' + 1)\Gamma(\beta_k')}{\Gamma(\alpha_k')\Gamma(\beta_k')\Gamma(\alpha_k'+\beta_k'+1)} \prod_{i:i<k} \frac{\Gamma(\alpha_i' + \beta_i')\Gamma(\alpha_i')\Gamma(\beta_i' + 1)}{\Gamma(\alpha_i')\Gamma(\beta_i')\Gamma(\alpha_i'+\beta_i'+1)} \prod_{i:i>k} \frac{\Gamma(\alpha_i' + \beta_i')\Gamma(\alpha_i')\Gamma(\beta_i')}{\Gamma(\alpha_i')\Gamma(\beta_i')\Gamma(\alpha_i'+\beta_i')}\\
     &= \frac{\alpha_k'}{\alpha_k'+\beta_k'} \prod_{i:i<k} \frac{\beta_i'}{\alpha_i'+\beta_i'},
\end{aligned}
\end{equation*} }
where we utilize $\delta_i = 1$ for $i<k$ and $\delta_i =0 $ for $i>k$. Similarly, for $s_k=2$ and $s_l=0$ for $l \ne k$, we have
{\small
\begin{equation*}
\begin{aligned}
\E[p_k^2|{\bf x}]  & = \frac{\Gamma(\alpha_k' + \beta_k')\Gamma(\alpha_k' + 2)\Gamma(\beta_k')}{\Gamma(\alpha_k')\Gamma(\beta_k')\Gamma(\alpha_k'+\beta_k'+2)} \prod_{i:i<k} \frac{\Gamma(\alpha_i' + \beta_i')\Gamma(\alpha_i')\Gamma(\beta_i' + 2)}{\Gamma(\alpha_i')\Gamma(\beta_i')\Gamma(\alpha_i'+\beta_i'+2)} \prod_{i:i>k} \frac{\Gamma(\alpha_i' + \beta_i')\Gamma(\alpha_i')\Gamma(\beta_i')}{\Gamma(\alpha_i')\Gamma(\beta_i')\Gamma(\alpha_i'+\beta_i')}\\
     &= \frac{(\alpha_k'+1)(\alpha_k')}{(\alpha_k'+\beta_k'+1)(\alpha_k'+\beta_k')} \prod_{i:i<k} \frac{(\beta_i'+1)(\beta_i')}{(\alpha_i'+\beta_i'+1)(\alpha_i'+\beta_i')}.
     \end{aligned}
\end{equation*} 
}

For the term $p_K$, applying \eqref{GKmoment} from Lemma \ref{moment} with $s=1$ and $s=2$, we respectively obtain
\begin{equation*}
    \E[p_K|{\bf x}]
     = \prod_{k=2}^{K-1} \frac{\Gamma(\alpha_k' + \beta_k')\Gamma(\beta_k' + 1)}{\Gamma(\beta_k')\Gamma(\alpha_k'+\beta_k'+1)}
     = \prod_{k=2}^{K-1} \frac{\beta_k'}{\alpha_k'+\beta_k'}
\end{equation*} 
and 
\begin{equation*}
    \E[p_K^2|{\bf x}] 
     = \prod_{k=2}^{K-1} \frac{\Gamma(\alpha_k' + \beta_k')\Gamma(\beta_k' + 2)}{\Gamma(\beta_k')\Gamma(\alpha_k'+\beta_k'+2)}
     = \prod_{k=2}^{K-1} \frac{(\beta_k'+1)\beta_k'}{(\alpha_k'+\beta_k'+1)(\alpha_k'+\beta_k')}.
\end{equation*}
\end{proof}

\subsection{Definition}

In this subsection, we present the definition of our proposed BCVI. For a data point ${\bf x} = (x_1,x_2,\dots,x_n)$, let $f({\bf x}|{\bf p})$ denote the conditional probability density function given ${\bf p}$, as specified in \eqref{datadist}, where $p_k$ represents the probability that the actual number of clusters is $k$. We further assume that the prior distribution of ${\bf p}$ follows either a Dirichlet or GD distribution, as described in the previous section. The BCVI is then defined as follows.

\begin{definition}
For $k = 2,3,\ldots, K$,
\bea \label{defBCVI}
\texttt{BCVI}(k) = \E[p_k|{\bf x}]
\ena
where $\E[p_k|{\bf x}]$ is computed according to either Corollary \ref{DirCor} or Corollary \ref{GDirCor}.
\end{definition}

Though any index may be taken as an underlying CVI in \eqref{idxratio}, in this work, we focus on testing our proposed BCVI based solely on the WI and the WP.  

\begin{remark}

\begin{enumerate}
    \item Since $p_k$ represents the probability that the actual number of clusters is $k$, we can use these probabilities to construct a confidence set. For instance, if $p_{i_1}+p_{i_2}+p_{i_3}$, we are confident that the actual number of clusters lies within the set $\{i_1,i_2,i_3 \}$.
    \item BCVI is only meaningful when the underlying CVI can provide a ranking of the optimal numbers of clusters. If the underlying CVI only indicates the best option, $p_k$ is not a valid probability.
\end{enumerate}

\end{remark}

\subsection{Mathematical properties}

In this subsection, we consider some properties of our BCVI. It is reasonable to assume that $\alpha_k$ is of the same order with respect to $n$ for all $k$.   

\subsubsection{Dirichlet prior}
By \eqref{defBCVI} and Corollary \ref{DirCor}, for $k=2,3,\ldots,K$, the BCVI is given by 
\bea \label{Dbcvi}
\texttt{BCVI}(k) = \frac{\alpha_k + nr_k({\bf x})}{\alpha_0+n}.
\ena

The following proposition analyzes the behavior of BCVI when $n$ is large according to the order of $\alpha_k$ in each situation.

\begin{proposition} \label{alphabeh}
For $k=2,3,\ldots,K$, the following holds.
\begin{enumerate}
    \item If $\max_k \alpha_k = o(n)$ as $n \rightarrow \infty$, then 
\beas
\lim_{n \rightarrow \infty}\texttt{BCVI}(k) = r_k({\bf x}).
\enas
\item If $\alpha_k = O(n)$, i.e. $\alpha_k/n \rightarrow c_k$ as $n \rightarrow \infty$ for some $c_k>0$ for all $k$, then 
\beas
\lim_{n \rightarrow \infty}\texttt{BCVI}(k) = \frac{c_k + r_k({\bf x})}{\sum_{j=2}^K c_j + 1}.
\enas
\item If $\alpha_k = O(g(n))$, $g(n)/n \rightarrow \infty$, and $\alpha_k/g(n) \rightarrow c_k$ as $n \rightarrow \infty$, then
\beas
\lim_{n \rightarrow \infty}\texttt{BCVI}(k) = \frac{c_k}{\sum_{j=2}^K c_j}.
\enas
\end{enumerate}

\end{proposition}

\begin{proof}

From \eqref{Dbcvi} and that $\alpha_k/n \rightarrow 0$, $\alpha_k \rightarrow c_k$, and $\alpha_k/g(n) \rightarrow c_k$ as $n \rightarrow \infty$ for all $k$, we obtain

\begin{equation*}
1. \lim_{n \rightarrow \infty}\texttt{BCVI}(k) =\lim_{n \rightarrow \infty}\frac{\alpha_k + nr_k({\bf x})}{\alpha_0+n} =  \lim_{n \rightarrow \infty} \frac{\frac{\alpha_k}{n} + r_k({\bf x})}{\frac{\sum_{j=2}^K \alpha_j}{n}+1}  \\
= r_k({\bf x}), 
\end{equation*}

\begin{equation*}
2. \lim_{n \rightarrow \infty}\texttt{BCVI}(k) =\lim_{n \rightarrow \infty}\frac{\alpha_k + nr_k({\bf x})}{\alpha_0+n} =  \lim_{n \rightarrow \infty} \frac{\frac{\alpha_k}{n} + r_k({\bf x})}{\frac{\sum_{j=2}^K \alpha_j}{n}+1} \\
= \frac{c_k + r_k({\bf x})}{\sum_{j=2}^K c_j + 1},
\end{equation*} 
and

\begin{equation*}
3. \lim_{n \rightarrow \infty}\texttt{BCVI}(k) =\lim_{n \rightarrow \infty}\frac{\alpha_k + nr_k({\bf x})}{\alpha_0+n} =  \lim_{n \rightarrow \infty} \frac{\frac{\alpha_k}{g(n)} + \frac{nr_k({\bf x})}{g(n)}}{\frac{\sum_{j=2}^K \alpha_j}{g(n)}+\frac{n}{g(n)}}  \\
= \frac{c_k}{\sum_{j=2}^K c_j}.
\end{equation*} 

\end{proof}

\begin{remark} \label{remarkDirichlet}
Theorem \ref{alphabeh} implies that if $\alpha_k$ is of a larger order than $n$, then the underlying index has no impact on BCVI when $n$ is large. Similarly, if $\alpha_k$ is of the same order as $n$, then the underlying index has the same impact on BCVI regardless of the sample size. Therefore, selecting $\alpha_k = o(n)$ such that $\alpha_k \rightarrow \infty$ as $n \rightarrow \infty$, such as $\alpha_k = O(\sqrt{n})$, makes the most sense. However, in cases where users have strong beliefs about the expected number of clusters, $\alpha_k = O(n)$ may also be chosen. 
\end{remark}

The next two propositions provide the properties of BCVI for some specific ${\bm \alpha}$.

\begin{proposition}
If $\alpha_k$, $k=2,3,\ldots,K$ are all equal, then BCVI and the underlying CVI, i.e., GI, are equivalent.
\end{proposition}

\begin{proof}
Assume that $\alpha_k = \alpha$ for all $k$. From \eqref{Dbcvi}, we have
\beas
\texttt{BCVI}(k)  = \frac{\alpha + nr_k({\bf x})}{(K-1)\alpha+n}
\enas
which is a function of $k$ only through $r_k(x)$ defined in \eqref{idxratio}. Since $r_k(x)$ is adjusted monotonously from the underlying index GI$(k)$, BCVI$(k)$ and GI$(k)$ yield exactly the same ranking of the preferred numbers of groups.
\end{proof}

\begin{proposition} \label{localProp}
Let $s\in \mathbb{N}$ and $i_1,i_2,\ldots, i_s \in \{2,3,\ldots,K\}$ be local peaks of GI$(k)$. If for any $j \in [s]$, $\alpha_{i_j-1}, \alpha_{i_j}, \alpha_{i_j+1}$ are equal, then $i_1,i_2,\ldots, i_s$ remain local peaks with respect to BCVI$(k)$.
\end{proposition}

\begin{proof}
For $j \in [s]$, assume that $\alpha_{i_j-1}= \alpha_{i_j}= \alpha_{i_j+1} = \alpha$. Then, by \eqref{Dbcvi}, we have
$\texttt{BCVI}_{i_j-1}  = \frac{\alpha + nr_{i_j-1}({\bf x})}{(K-1)\alpha+n}$, $\texttt{BCVI}_{i_j}  = \frac{\alpha + nr_{i_j}({\bf x})}{(K-1)\alpha+n}$, and $\texttt{BCVI}_{i_j+1}  = \frac{\alpha + nr_{i_j+1}({\bf x})}{(K-1)\alpha+n}$. Since $i_j$ is a local peak, $r_{i_j}$ is greater than both $r_{i_j-1}$ and $r_{i_j+1}$. This implies that $\texttt{BCVI}$ still has a local peak at $i_j$. 
\end{proof}

\subsubsection{Generalized Dirichlet prior}

By \eqref{defBCVI} and Corollary \ref{GDirCor}, the BCVI is 
\bea \label{GDbcvi}
\texttt{BCVI}(k) = \frac{\alpha_k + nr_k(x)}{\alpha_k+\beta_k + \sum_{i=k}^K nr_i(x) } \prod_{i:i<k} \frac{\beta_i + \sum_{j=i+1}^K nr_j(x)}{\alpha_i+\beta_i + \sum_{j=i}^K nr_j(x)}
\ena
for $k=2,3,\ldots,K-1$, and 
\bea \label{GDbcvi2}
\texttt{BCVI}(K) = 
\prod_{k=2}^{K-1} \frac{\beta_k + \sum_{i=k+1}^K nr_i(x)}{\alpha_k+\beta_k + \sum_{i=k}^K nr_i(x)}.
\ena

The following proposition examines the behavior of BCVI when $n$ is large relative to the orders of $\alpha_k$ and $\beta_k$ in each scenario.

\begin{proposition} \label{alphabeh}
\begin{enumerate}
    \item If $\max_k \alpha_k = o(n)$ and $\max_k \beta_k = o(n)$ as $n \rightarrow \infty$, then for $k=2,3,\ldots,K$, 
\beas
\lim_{n \rightarrow \infty}\texttt{BCVI}(k) =r_k(x).
\enas

\item If $\alpha_k = O(n)$ and $\beta_k = O(n)$, i.e. $\alpha_k/n \rightarrow c_k$ and $\beta_k/n \rightarrow d_k$ as $n \rightarrow \infty$ for some $c_k,d_k>0$ for all $k$, then for $k=2,3,\ldots,K-1$,
\beas
\lim_{n \rightarrow \infty}\texttt{BCVI}(k) 
= \frac{c_k +  r_k(x)}{c_k + d_k + \sum_{i=k}^K r_i(x)} \prod_{i:i<k} \frac{d_i + \sum_{j=i+1}^K r_j(x)}{c_i + d_i + \sum_{j=i}^K r_j(x)},
\enas
and
\beas
\lim_{n \rightarrow \infty}\texttt{BCVI}(K) =\prod_{k=2}^{K-1} \frac{d_k + \sum_{i=k+1}^K r_i(x)}{c_k + d_k + \sum_{i=k}^K r_i(x)} .
\enas

\item If $\alpha_k = O(g(n))$ and $\beta_k = O(g(n))$, $g(n)/n \rightarrow \infty$, and $\alpha_k/g(n) \rightarrow c_k$ and $\beta_k/g(n) \rightarrow d_k$ as $n \rightarrow \infty$, then for $k=2,3,\ldots,K-1$,
\beas
\lim_{n \rightarrow \infty}\texttt{BCVI}(k) = 
\frac{c_k }{c_k + d_k }  \prod_{i:i<k} \frac{d_i}{c_i + d_i} ,
\enas
and
\beas
\lim_{n \rightarrow \infty}\texttt{BCVI}(K) =\prod_{k=2}^{K-1} \frac{d_k }{c_k + d_k} .
\enas

\end{enumerate}

\end{proposition}

\begin{proof}

Using \eqref{GDbcvi} and \eqref{GDbcvi2} and the assumptions 1, 2, and 3 in the statement, we obtain the following.
\begin{enumerate}
    \item  Dividing both the numerator and the denominator of \eqref{GDbcvi} by $n^{k-1}$ for $k = 2,\ldots,K-1$ and of \eqref{GDbcvi2} by $n^{K-2}$, we have
\begin{equation*}
\begin{aligned}
   \lim_{n \rightarrow \infty}\texttt{BCVI}(k) 
   &= \lim_{n \rightarrow \infty} \frac{\frac{\alpha_k}{n} + r_k(x)}{\frac{\alpha_k}{n}+\frac{\beta_k}{n} + \sum_{i = k}^K r_i(x)} \prod_{i:i<k} \frac{\frac{\beta_i}{n} + \sum_{j = i+1}^K r_j(x)}{\frac{\alpha_i}{n} + \frac{\beta_i}{n} + \sum_{j = i}^K r_j(x)} \\
   &= \frac{r_k(x)}{\sum_{i = k}^K r_i(x)} \prod_{i:i<k} \frac{\sum_{j = i+1}^K r_j(x)}{\sum_{i = k}^K r_i(x)} = r_k(x),
\end{aligned}
\end{equation*} 
and 
\begin{equation*}
    \begin{aligned}
        \lim_{n \rightarrow \infty}\texttt{BCVI}(K) &= \lim_{n \rightarrow \infty} \prod_{k=2}^{K-1} \frac{\beta_k + \sum_{i=k+1}^K nr_i(x)}{\alpha_k+\beta_k + \sum_{i=k}^K nr_i(x)} = \lim_{n \rightarrow \infty} \prod_{k=2}^{K-1} \frac{\frac{\beta_k}{n} + \sum_{i=k+1}^K r_i(x)}{\frac{\alpha_k}{n}+ \frac{\beta_k}{n} + \sum_{i=k}^K r_i(x)} \\
         &= \prod_{k=2}^{K-1} \frac{\sum_{i=k+1}^{K} r_i(x)}{\sum_{i=k}^{K} r_i(x)} = r_K(x).
    \end{aligned}
\end{equation*}

        \item  Dividing both the numerator and denominator of \eqref{GDbcvi} by $n^{k-1}$ for $k = 2,\ldots,K-1$ and of \eqref{GDbcvi2} by $n^{K-2}$, we obtain

\begin{equation*}
    \begin{aligned}
         \lim_{n \rightarrow \infty}\texttt{BCVI}(k)  
         &= \lim_{n \rightarrow \infty} \frac{\frac{\alpha_k}{n} + r_k(x)}{\frac{\alpha_k}{n}+\frac{\beta_k}{n} + \sum_{i = k}^K r_i(x)} \prod_{i:i<k} \frac{\frac{\beta_i}{n} + \sum_{j = i+1}^K r_j(x)}{\frac{\alpha_i}{n} + \frac{\beta_i}{n} + \sum_{j = i}^K r_j(x)} \\
         &= \frac{c_k +  r_k(x)}{c_k + d_k + \sum_{i=k}^K r_i(x)} \prod_{i:i<k} \frac{d_i + \sum_{j=i+1}^K r_j(x)}{c_i + d_i + \sum_{j=i}^K r_j(x)},
    \end{aligned}
\end{equation*} 
and 
\begin{equation*} 
    \begin{aligned}
        \lim_{n \rightarrow \infty}\texttt{BCVI}(K)  
        = \lim_{n \rightarrow \infty} \prod_{k=2}^{K-1} \frac{\frac{\beta_k}{n} + \sum_{i=k+1}^K r_i(x)}{\frac{\alpha_k}{n} + \frac{\beta_k}{n}+ \sum_{i=k}^K r_i(x)} 
        = \prod_{k=2}^{K-1} \frac{d_k + \sum_{i=k+1}^K r_i(x)}{c_k + d_k + \sum_{i=k}^K r_i(x)}.
    \end{aligned}
\end{equation*}

\item Dividing both the numerator and the denominator of \eqref{GDbcvi} by $g^{k-1}(n)$ for $k = 2,\ldots,K-1$, and dividing both the numerator and denominator of \eqref{GDbcvi2} by $g^{K-2}(n)$, we obtain
\begin{equation*}
    \begin{aligned}
        \lim_{n \rightarrow \infty}\texttt{BCVI}(k) 
         &= \lim_{n \rightarrow \infty} \frac{\frac{\alpha_k}{g(n)} + \frac{nr_k(x)}{g(n)}}{\frac{\alpha_k}{g(n)}+\frac{\beta_k}{g(n)} + \frac{n}{g(n)}\sum_{i = k}^K r_i(x)} \prod_{i:i<k} \frac{\frac{\beta_i}{g(n)} + \frac{n}{g(n)}\sum_{j = i+1}^K r_j(x)}{\frac{\alpha_i}{g(n)} + \frac{\beta_i}{g(n)} + \frac{n}{g(n)}\sum_{j = i}^K r_j(x)} \\
         &= \frac{c_k }{c_k + d_k } \prod_{i:i<k} \frac{d_i }{c_i + d_i},
    \end{aligned}
\end{equation*} 
and
\begin{equation*}
    \begin{aligned}
        \lim_{n \rightarrow \infty}\texttt{BCVI}(K)  
        &= \lim_{n \rightarrow \infty} \prod_{k=2}^{K-1} \frac{\frac{\beta_k}{g(n)} + \frac{n}{g(n)}\sum_{i=k+1}^K r_i(x)}{\frac{\alpha_k}{g(n)} + \frac{\beta_k}{g(n)}+ \frac{n}{g(n)}\sum_{i=k}^K r_i(x)}= \prod_{k=2}^{K-1} \frac{d_k }{c_k + d_k}.
    \end{aligned}
\end{equation*}

\end{enumerate}

\end{proof}

It is evident that we can interpret the above proposition similarly to Remark \ref{remarkDirichlet} regarding the Dirichlet case.

\section{Experimental results} \label{sec:exp}

In this section, we demonstrate the benefits of our proposed BCVI by comparing it to the underlying indices and several other existing indices as listed in Section \ref{sec:background}. Our experiment is divided into two distinctive parts, each described in a separate subsection: artificial datasets and real-world datasets, including MRI datasets. We test our BCVI using two clustering algorithms: K-means and FCM with a fuzziness parameter of $2$ for hard and soft clustering, respectively. For the results presented in this entire section, we run either K-means or FCM for a total of 20  rounds and select the one with the smallest objective function. This is to make sure that the result is appropriate for testing the CVIs.
To facilitate our experiment, we utilize our dedicated R package called “UniversalCVI” \cite{UniversalCVI} within the RStudio environment \cite{Rstudio}. The ``Hvalid'' and ``FzzyCVIs'' functions are applied to compute all the indices for K-means and FCM, respectively.

\begin{table}[H]
\centering
\caption{Accuracy of clustering algorithms on labeled datasets}
\centering\includegraphics[width=16cm]{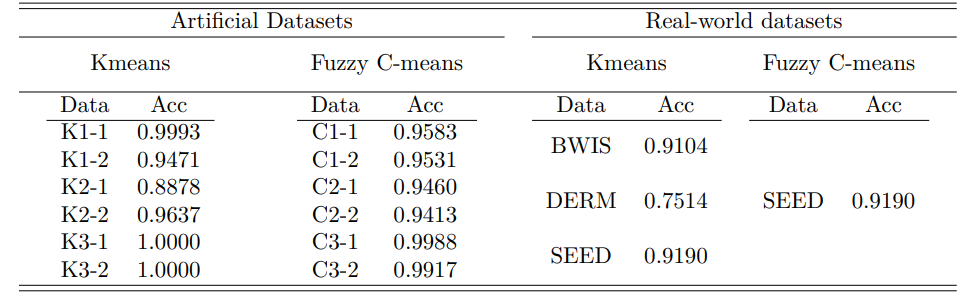}
\label{Acctab}
\end{table}

\begin{table}[H]
\centering
\caption{Dirichlet prior parameters}
\centering\includegraphics[width=16cm]{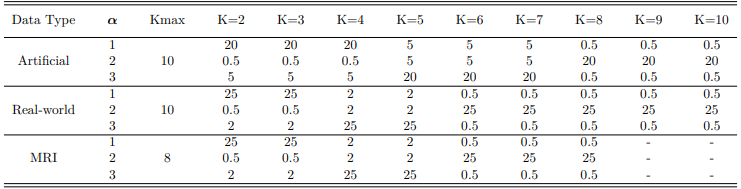}
\label{alphatable}
Note: The parameters shown in this table will be multiplied by $\sqrt{n}$ where $n$ is the number of data points.
\end{table}

The artificial datasets (K1-1 to K3-2, C1-1 to C3-2) are newly generated from the Gaussian distribution and the uniform distribution, except for K1-2 and C1-2, which are from \cite{Benchmarksdata2019}. The datasets are intentionally named K and C to correspond with K-means and FCM, respectively. The real-world datasets include BWIS, DERM, and SEED from UCI \cite{UCI2013}, and MRI brain tumor datasets (TUMOR1 from \cite{mri1}, and the remainder from \cite{mri234}). Since the WP underlying index is not directly applicable to big data, we resize the MRI images to 85×85 pixels before applying the BCVI while implementing the other existing CVIs to the full images. The selection of the main clustering algorithm on each dataset is based on accuracy when setting the number of clusters to be the actual number of groups. To ensure that the datasets are appropriate for testing the performance of CVIs, we first verify that the main clustering algorithms are applicable to them. We check the accuracy of K-means and FCM on those labeled datasets, as shown in Table \ref{Acctab}, using the `AccClust' function in the package `UniversalCVI'. Note that we intend to consider only the datasets with 75\% accuracy or above. This criterion is set because if the main clustering algorithm is unable to find the correct groups, then it is not suitable for testing the performance of CVIs.

As previously mentioned, our experiment for the proposed BCVI is conducted using only the Dirichlet prior. Therefore, we must first set its parameters according to \eqref{Dprior}. The BCVI results presented in this section are computed based on the parameters outlined in Table \ref{alphatable}. It is important to note that for each dataset, we propose three parameter options corresponding to three different scenarios: when the user prefers a small (${\bm \alpha}_1$), large (${\bm \alpha}_2$), and moderate (${\bm \alpha}_3$) number of groups, respectively. This approach is advantageous for users who have an approximate idea of how many clusters they expect.

\subsection{Artificial datasets}
In this subsection, we assess the performance of our BCVI on artificial datasets categorized into three distinct cases, as outlined below, in order to highlight the advantages of our proposed Bayesian approach. These datasets comprise both benchmark datasets and simulated datasets, as mentioned previously.

\begin{list}{}{} \label{datagroup}
\item{\bf Case 1: }{The underlying index incorrectly detects the true number of groups}
\item{\bf Case 2: }{The underlying index originally leads to one of the secondary options}
\item{\bf Case 3: }{The underlying index correctly detects the true number of groups, but users seek a secondary option because the optimal one is either too small or too large}
\end{list}

\begin{figure*}[h]
\centering\includegraphics[width=14cm]{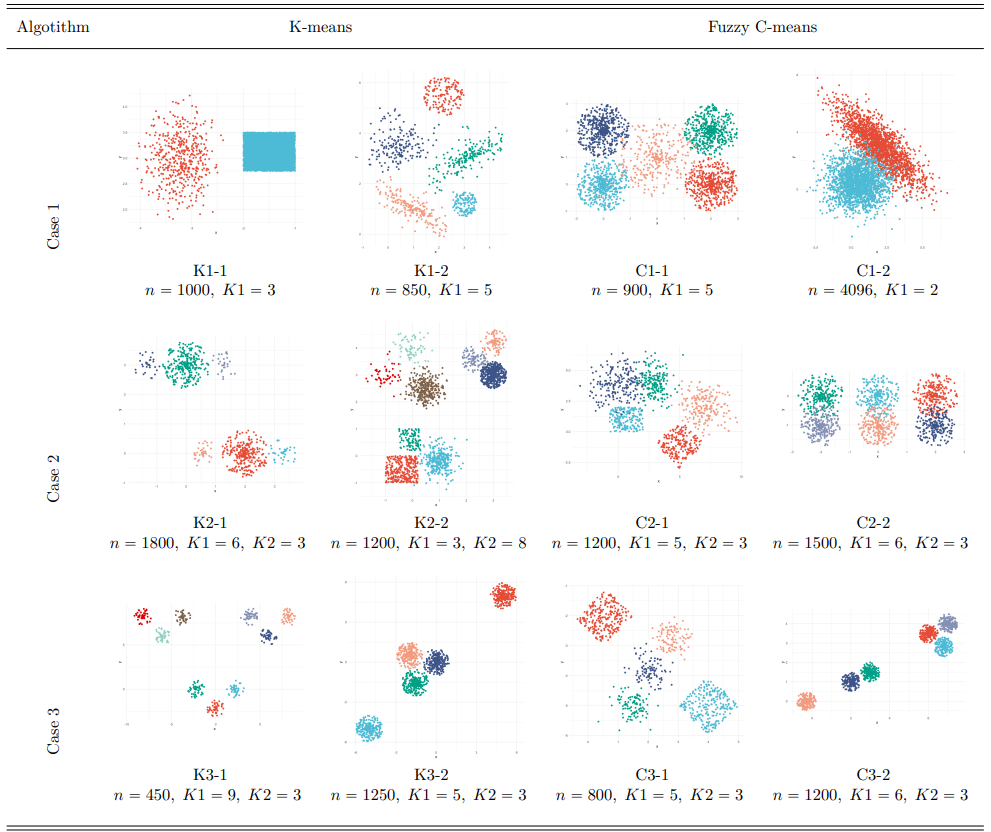}
\text{Note: $n$, $K1$, and $K2$ are the number of data points, the first and secondary options for the number}
\text{of clusters, respectively.}
\caption{Artificial datasets}
\label{ArfData}
\end{figure*}

\begin{table}[H]
\centering
\caption{Hard BCVI on artificial datasets}
\centering\includegraphics[width=16cm]{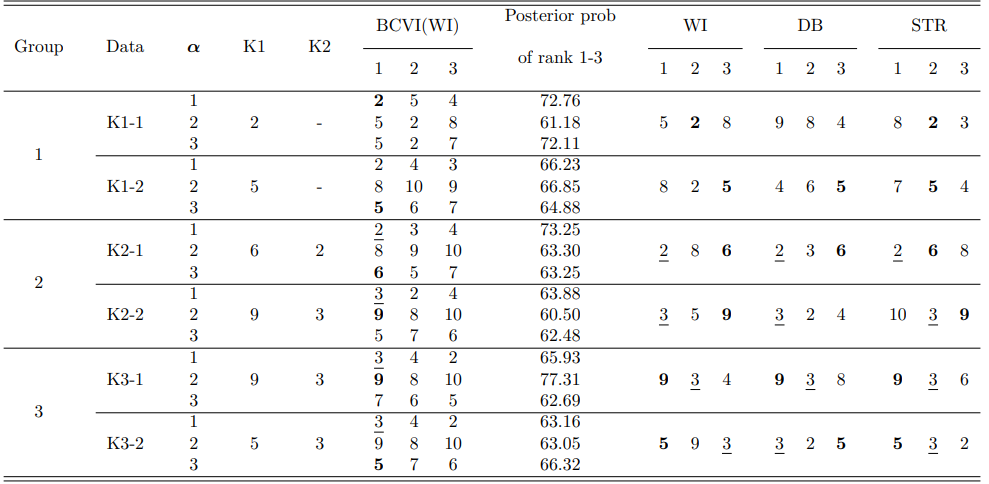}
\label{kmeansresults}
Note: K1 and K2 (if exists) are the optimal and the second optimal numbers of clusters, respective. The CVIs results are bold for K1 and underlined for K2 (these also apply to all the tables below). 
\end{table}

\begin{table}[H]
\centering
\caption{Soft BCVI on artificial datasets}
\centering\includegraphics[width=15cm]{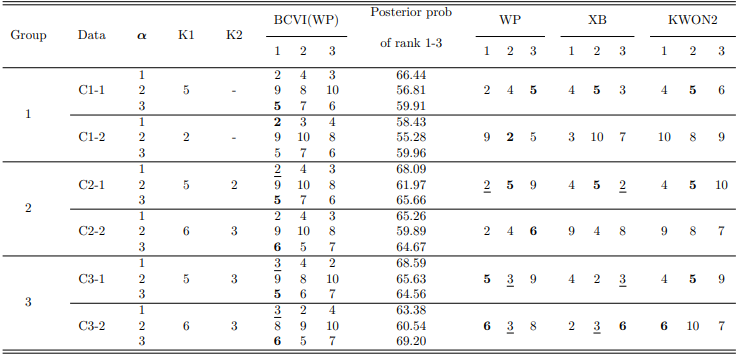}
\label{fcmresults}
\end{table}

The 12 datasets are depicted in Figure \ref{ArfData}: six for hard and soft CVIs each, categorized into the three previously mentioned cases. The parameters are configured for hypothetical scenarios based on users' preferences: two to four groups (${\bm \alpha}_1$), eight to ten groups (${\bm \alpha}_2$), and five to seven groups (${\bm \alpha}_3$). Table \ref{kmeansresults} and \ref{fcmresults} display the first three ranks for the final number of clusters according to each CVI, where the true optimal and second optimal options are bold and underlined, respectively. The parameters for BCVI are determined according to Table \ref{alphatable}, with the rationale explained thereafter. From Table \ref{kmeansresults} and \ref{fcmresults}, for the hard and soft clustering results, respectively, in cases 1 and 2, all the indices incorrectly detect the number of clusters. However, when applying our BCVI with the true number of clusters falling into the correct parameter scenario, BCVI provides the correct decision. This implies that the underlying index exhibits a local peak at the true number of clusters. In case 3, we illustrate when the underlying index can accurately detect the number of clusters, but that number is either too large or too small to implement. By adjusting the parameters to align with the requirements, BCVI guides the final decision toward a secondary choice. Additionally, we calculate the posterior probability of the first three ranks, as shown in Table \ref{kmeansresults} and \ref{fcmresults}. It indicates that we can be approximately 60\% to 80\% confident that the actual number of clusters falls within the first three ranks.

\subsection{Real-world and MRI datasets}

\begin{figure*}[h]
\centering\includegraphics[width=13cm]{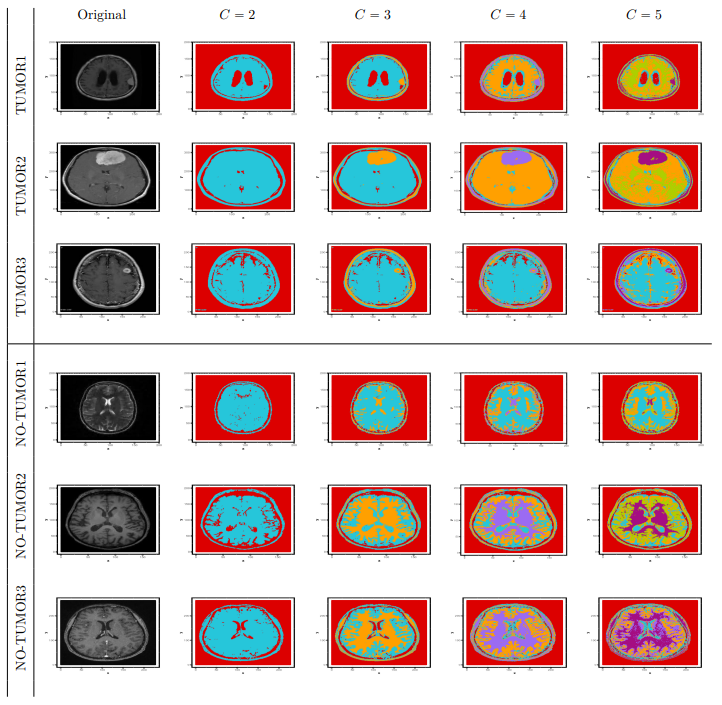}
\caption{MRI Datasets}
\label{fig:image}
\end{figure*}

In this subsection, we apply our BCVI to real-world datasets: BWIS, DERM, and SEED (Table \ref{rwresults}), along with MRI brain tumor images shown in Figure \ref{fig:image}. For hard clustering, we consider DB, STR, and WI indices, while for soft clustering, we consider XB, KWON2, and WP indices. We intentionally select datasets where our underlying indices incorrectly detect the true numbers of classes. This is to underscore that our BCVI can rectify incorrect estimations when users possess prior knowledge about their expected number of classes. Table \ref{rwresults} and \ref{MRIresults} display the first three ranks for the final number of clusters based on each CVI, where the true optimal option is bolded. The parameters for BCVI are determined according to Table \ref{alphatable}, with the rationale explained thereafter.

\begin{table}[h]
\centering
\caption{Hard BCVI on real-world datasets}
\centering\includegraphics[width=16cm]{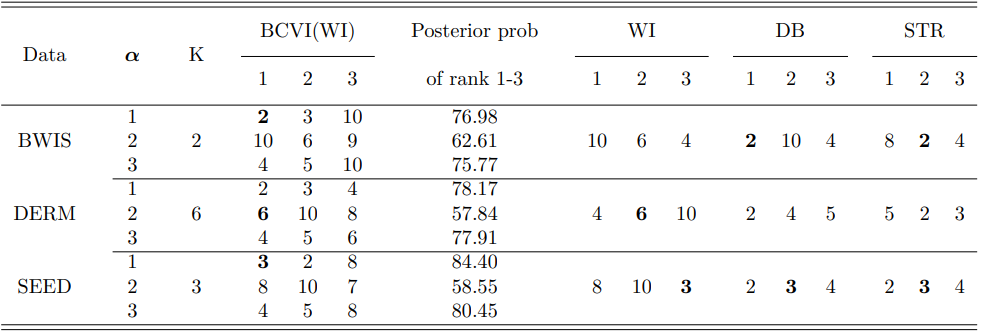}
\label{rwresults} 
\end{table}

\begin{table}[h]
\centering
\caption{Soft BCVI on real-world and MRI datasets}
\centering\includegraphics[width=16cm]{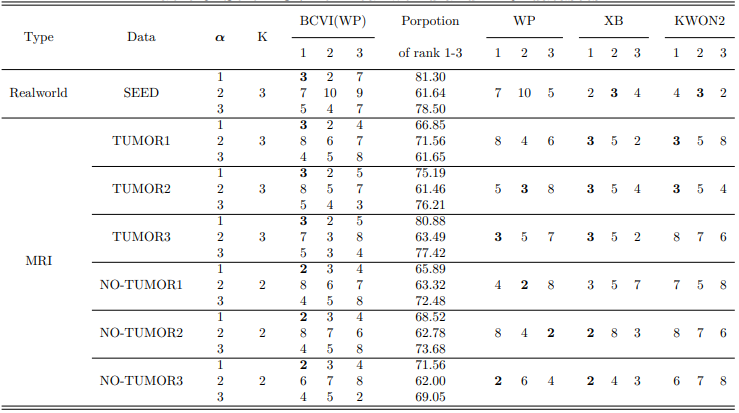}
\label{MRIresults}
\end{table}

For the real-world datasets, except MRI images, BCVI has the capability to correct erroneous underlying CVI decisions when the parameters are appropriately configured. Additionally, we can have approximately 80\% confidence that the true number of clusters falls within the first three ranks of BCVI.

To detect tumors in MRI images, it is evident from Figure \ref{fig:image} that three colors are most appropriate for images with tumors, while two colors suffice for those without tumors. Additionally, having too many colors can needlessly complicate the images. Therefore, we set the parameters to be large at two and three groups and limit the maximum number of clusters considered to be eight. From Table \ref{MRIresults}, it is apparent that none of the existing indices can accurately detect the correct numbers of clusters for all six images. However, by applying BCVI with ${\bm \alpha}_1$ from Table \ref{alphatable} and using the WP underlying index, we successfully identify the numbers of clusters for all six images. According to Proposition \ref{localProp}, this suggests that the underlying WP index possesses correct local peaks for all six MRI images.  

\begin{figure*}[h!]
\centering\includegraphics[width=11cm]{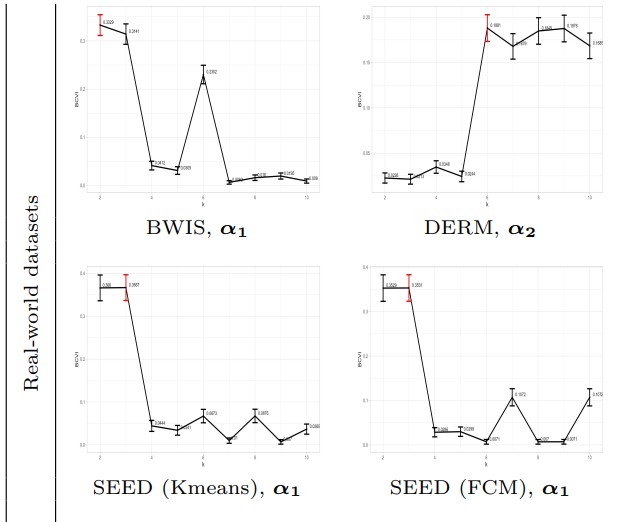}
\caption{BCVI values on real-world datasets}
\label{fig:ggraphrw}
Note: We show only ${\bm \alpha}$ values from Table \ref{alphatable} that yield our intended numbers of groups.
\end{figure*}

\begin{figure*}[h!]
\centering\includegraphics[width=13cm]{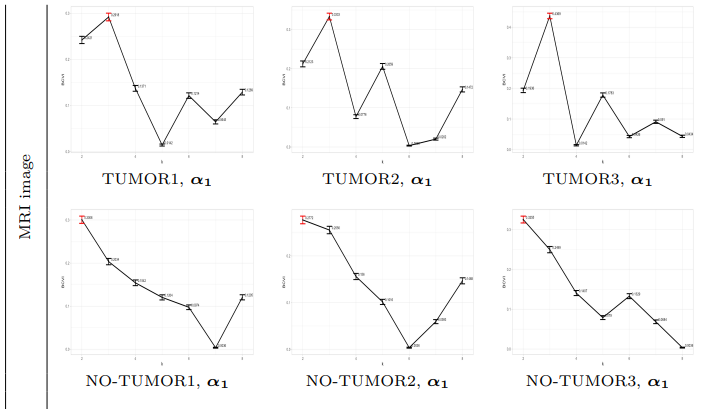}
\caption{BCVI values on MRI image dataset}
\label{fig:ggraphmri}
Note: We show only ${\bm \alpha}$ values from Table \ref{alphatable} that yield our intended numbers of groups.
\end{figure*}

Additionally, we provide plots of the BCVI values with two standard deviation error bars for the ${\bm \alpha}$ values corresponding to our intended numbers of groups in Figures \ref{fig:ggraphrw} and \ref{fig:ggraphmri}. These standard deviations are computed directly from Corollary \ref{DirCor}. This aims to demonstrate the variation of BCVI values concerning $k$. We present examples of these plots solely for the real-world case to illustrate that the statistical inference of our BCVI can be further analyzed in future works.

\section{Conclusion} \label{sec:conclusion}

In this work, we propose a new concept that directly applies the Bayesian framework to CVIs. This connection is unique and has not been explored before in the literature. Specifically, we introduce BCVI, which can be used in conjunction with any existing CVIs. This is advantageous for users who have prior knowledge of the expected number of clusters in their datasets. Although we only test BCVI with K-means and FCM, it can be applied to any clustering algorithms compatible with any CVIs. BCVI calculates the posterior probability of the actual number of clusters, with the prior distribution being either Dirichlet or GD. The main features of our BCVI are as follows:

\begin{enumerate}
    \item {\bf Novel and unique concept:} BCVI allows users to blend their knowledge with a dataset's pattern to identify the final number of clusters.
    \item {\bf Flexibility:} BCVI allows users to flexibly set parameters according to their needs and select any clustering algorithms and underlying CVIs of their choice. 
\end{enumerate}

We present the results of BCVI together with WI and WP, the underlying CVIs, and four other existing CVIs, namely DB, STR, XB, and KWON2. While this is not a comparative test due to the new concept, the results affirm that BCVI offers several primary advantages:
\begin{enumerate}
    \item {\bf Correcting erroneous results:} BCVI can lead to the correct number of clusters in cases where the underlying CVI is incorrect. However, this requires users to select appropriate parameters based on their knowledge.
    \item {\bf Providing alternative options:} BCVI can suggest alternative suboptimal numbers of clusters if the optimal one is not suitable for users in their context.
\end{enumerate}
These advantages are especially advantageous when the expected range is definite. For example, we are aware that a few colors are enough for MRI image pre-processing. 
The limitations of BCVI are that: 
\begin{enumerate}
    \item It relies on the quality of underlying indices.
    \item It is only effective when underlying indices are present, providing meaningful options for ranking local peaks for the final number of clusters.
\end{enumerate}

Future research directions include studying the statistical inference of BCVI, exploring alternative prior distributions, testing BCVI with additional clustering algorithms and underlying CVIs, and applying it to diverse real-world applications.

\end{document}